\documentclass[a4paper, 11pt]{article}

\usepackage[]{graphicx}
\usepackage{setspace}
\usepackage{tocloft}
\usepackage{url}
\usepackage{amsmath,amsfonts}
\usepackage[lined]{algorithm2e}
\usepackage[vmargin=2cm,hmargin=2cm]{geometry}
\usepackage{color}

\newenvironment{proof}[1][Proof]{\begin{trivlist}
\item[\hskip \labelsep {\bfseries #1}]}{\end{trivlist}}

\title{The Color Graph-based Wavelet Transform with Perceptual Information}
\author{M. Malek,  D.  Helbert, P. Carr\'e\\ Xlim, UMR CNRS 7252, University of Poitiers, France}

\cftpagenumbersoff{figure}
\cftpagenumbersoff{table} 

\newcommand{\argmin}{\arg\!\min}
\newcommand{\R}{\mathbb{R}}
\newtheorem{proposition}{Proposition}

\begin{document}

\maketitle 

\begin{abstract}
In this paper, we propose a numerical strategy to define a multiscale analysis for color and multicomponent images based on the representation of data on a graph. Our approach consists in computing the graph of an image using the psychovisual information and analysing it by using the spectral graph wavelet transform. We suggest introducing color dimension into the computation of the weights of the graph and using the geodesic distance as a mean of distance measurement. 
We thus have defined a wavelet transform based on a graph with perceptual information by using the CIELab color distance. 

This new representation is illustrated with denoising and inpainting applications. Overall, by introducing psychovisual information in the graph computation for the graph wavelet transform we obtain very promising results. Therefore results in image restoration highlight the interest of the appropriate use of color information.
\end{abstract}


{\noindent \footnotesize{\bf For all correspondence to}: David Helbert, E-mail:  \url{david.helbert@univ-poitiers.fr} }

\section{Introduction}
The  use of the multiscale analysis transform in the field of image processing efficiently analyses the visual information of  an image allowing the localization of the information in the time and frequency domains. The classical schemes of the wavelet decomposition are not specially defined for a color image and they are usually marginalized: it consists in applying the transform to each component separately. But that could lead to undesirable effects on the color.
 We can still find in the literature several frameworks that propose multiscale vectorial techniques. For example in our team we have developed an approach that uses the analytic wavelet which is principally based on the principle of the monogenic signal \cite{ar:mong}. This allows the non-marginalized processing of color and provides an analysis of the separate orientations of the information but this representation is based on the signal principle.

Recently many fundamental based-graph tools in the field of signal processing have emerged and have allowed to extend the classical multiscale transform to an irregular domain \cite{ar:Nar}, \cite{ar:Gav}. In \cite{ar:cof} the diffusion wavelets are based on the definition of the diffusion operator and the localization at each scale depends on the dyadic powers of this operator. The graph wavelet transform proposed by Crovella and Kolaczyk in \cite{ar:crov} is an example of the design in the vertex domain, when the construction of the wavelet function $\psi_{j,n}$ is designed using the geodesic distance and is localized according to the source $n$ and the scale $j$.
 In \cite{ar:narang1} Narang and Ortega proposed a local two channel critically sampled filter bank on graphs that requires the combination of different processing blocks such as upsampling on graphs (select of vertex and graph reduction) \cite{ar:narang3} and filtering. For this, the authors introduced the technique of graph partition based on the max cut on the graph and the $\mbox{k-RGB}_s$ as a graph reduction method. In order to ensure the recovery, they proposed in \cite{ar:narang2} a perfect reconstruction of the two channel wavelet filter banks. So they take advantages of the properties of the bipartite graphs to define the low pass and high pass filters. They also need to decompose the original graph into a series of bipartite subgraphs.
In \cite{ar:fram} Shuman {\it et al} extend the Laplacian pyramid transform on the graph domain. For this, they suggest a general way to select vertices based on the polarity of the largest eigenvector, then they use the Kron reduction technique \cite{ar:kron} followed by the step of sparsification graph \cite{ar:sparsi}. That provides a perfect reconstruction transform but the dimension of the output is different than the input signal. 
None of these methods takes into account the perceptual information in the based-graph wavelet decomposition.

In \cite{ar:hammond}  the Spectral Graph Wavelet Transform (SGWT) is based on the representation of the wavelet coefficients in the spectral domain. 
It turns out that it is more flexible and technically easier to implement. Furthermore, graph construction is a significant milestone for the application of this transform. 

In our work, we focus on the properties of the SGWT to define an approach for color analysis.  In practice, it is increasingly necessary to adapt strategies with human vision and it is thus becoming desirable to take into consideration the psychovisual system in the design of the multi-scale analysis tools. The idea is to use the geodesic distance to compute the weight of pairwise pixels.  This principle helps highlight the perceptual properties through the implementation of the $\Delta E_{2k}$ color difference\cite{ar:zhang,ar:luo,ar:vienna}. To the best of our knowledge, it is the first suggestion of a wavelet transform based on perceptual graphs.    

In this paper, we review a few methods to construct graphs from image data and we give an overview of the spectral graph theory and the main characteristics of the SGWT. We then suggest a color representation by inserting the color dimension into the measured distance, and more generally we consider the possibility of a psychovisual representation by the insertion of the  $\Delta E_{2k}$ color difference in the graph computation. 
To interpret the influence of the structure of the graph on the wavelet scale coefficients, we propose a general way to understand the behavior of the decomposition using the notion of quadratic form.
We use the highlight of the color perceptual features associated with this new representation through the application of image denoising and we develop a method of inpainting on a graph from the wavelet coefficients.
    
In section \ref{sec:2} we recall some elements of the weighted graph and spectral graph theory. In section \ref{sec:33} we develop our method by inserting the color dimension into the distance measured and finally in section \ref{sec:5} we develop a color image denoising method and an inpainting strategy.

\section{Preliminary Notations and Definitions} \label{sec:2}
This section provides the notations and the definition of a graph used throughout the paper, we also present how one can represent discrete data defined on regular or irregular domains with a weighted graph.
To begin with, we will review some concepts that are introduced or presented for example in \cite{ar:chung,spielman2012spectral,ar:hammond}.
\subsection{Definition of a graph}
Any discrete domain can be represented by a weighted graph. We note $G=(V,E,w)$ a loopless, undirected and weighted graph where $V$ is a finite set comprised of elements called vertices and $E$ is a subset of edges $V \times V$. An edge $(n,m)\in E$ connects two neighbor vertices $n$ and $m$. 
The neighborhood of a vertex $n$ is noted:
\begin{align}
\mathcal N(n)=\left\{ m\in V\backslash \{n\}:(m,n)\in E\right\}.
\end{align}

The weight $w_{mn}$ of an edge $(m, n)$ can be defined with a function $w:V\times V  \longrightarrow \mathbb{\mathbb{R}}^+$ such that $w(m, n)=w_{mn}$ if $(m,n)\in E$ and $w(m,n)=0$ otherwise. 

Let $\mu_{mn}$ be a distance computed between the vertices $m$ and $n$. In practice, it is difficult to establish a general rule to define this distance. The methods used are essentially based on local and non-local comparisons of features (see more details in \cite{ar:buades}).  Obviously, the weight function is represented by a decreasing function as follows:
\begin{align}\label{eq:we}
w(m,n)&=\exp(\frac{- \mu_{mn} ^2}{\sigma^2}) ,\; \sigma > 0,
\end{align}  
where $\sigma$ is the parameter that adjusts the influence on the pixel's neighborhood in the distance computation: the weight $w(m,n)$ of edge $(m,n)$ increases when $\sigma$ increases.

See more examples of weight functions in \cite{ar:gradly}.

We note $\textbf A$ the weighted adjacency matrix of size $N \times N$ with $N$ the total number of vertices:
\begin{eqnarray}\label{eq:A}
\textbf{A}_{m,n}=\left\{ 
\begin{array}{l}
w(m,n) \;\mbox{if}\; (m,n)\in E\\ 0 \;\mbox{otherwise.}
\end{array}
\right. 
 \end{eqnarray}

The degree $\delta_m$ of a vertex $m$ is defined as the sum of the weights of all the adjacent vertices to $m$ in the graph: 
\begin{align}\label{eq:D}
\delta_m=\sum_{n\neq m}\textbf{A}_{m,n}.
\end{align}
The degree matrix $\textbf{D}$ is the diagonal matrix of the degrees $\delta_m$ of $G$.

One of the issues encountered when dealing with the graph is the graph construction. There are two steps to build the graph, the first one consists in the definition of the graph topology which focuses primarily on the link between all regions. Then, in the second step the weights are computed to measure the similarities between these regions, as discussed below.


Depending on the structure of the data, there are different ways to construct the graph. The constructed graph has to be adapted to represent the discrete domain of the data but there is no general rule of data representation and that basically depends on the application \cite{ar:hein}. 
Indeed, in the case of an irregular domain, one can use the $v$-neighborhood graph noted $G_v$ \cite{ar:jarom} \cite{ar:maier}. Its principle is based on the notion of neighborhoods of a vertex $n$ that determine the set of vertices $N_v (\mu)$ whenever the distance $\mu$ is  lower than or equals the threshold parameter $v$.
The distance $\mu$ helps measure the similarities between the features vectors associated with the pixels. The simplest distance is the Euclidean distance noted {\em ED}.

In the case of a regular domain, the spatial organization of data is well-known, thus appropriate measurements are used in order to design a regular mesh such as the infinity norm and the unitary norm.
An alternative approach used in image processing consists in representing data structure in the shape of regions which are represented thereafter by a Region Adjacency Graph (see examples of partitions in \cite{ar:cousty,ar:arbel,ar:vincent}). 
In our approach we work directly on the pixels and to analyze the constructed graph we propose to work in the spectral graph domain. In the following we give some definition and notation of the spectral graph theory. 

\subsection{The spectral graph theory}  \label{sec:3}
In this section, we recall some basic definitions of the spectral graph theory \cite{ar:chung,spielman2012spectral}, and we also provide some details about the Spectral Graph Wavelet Transform (SGWT) \cite{ar:hammond}.

Here, we consider an undirected weighted graph $G$ with $N$ vertices. We can associate a vector in $\R^N$ on any scalar valued function $f$ defined on the vertices of the graph. In the case of a color image, this vector is the three components of the image.

\subsubsection{The based-graph Fourier transform}
This amounts to adapting the Fourier transform for the graph domain. Let $G=(E,V,w)$ be a weighted graph, the non-normalized Laplacian graph $\mathbf L$ is:
\begin{align}
\mathbf{L}=\mathbf{D}-\mathbf{A},
\end{align}  
where $\mathbf A$ and $\mathbf D$ denote respectively the adjacency matrix and the degree matrix computed according to equations \ref{eq:A} and \ref{eq:D}.

Thanks to the spectral analysis, it is established that the Fourier basis is associated with the eigenvectors of the Laplacian operator.
By analogy, the eigenvectors of $\mathbf{L}$ noted by $\chi_l$ with $l=0,\cdots,N-1$ constitute an orthonormal basis for the graph spectral domain:
\begin{align}
\mathbf{L}\chi_l&=\lambda_l \chi_l,
\end{align} 
with $\{(\lambda_l,\chi_l)\}_{l=0,1,...,N-1}$ the eigenvalue and eigenvector pairs of the graph Laplacian $\mathbf L$.

Let $f \in \R^N$ be a based-graph signal. Consequently, the based-graph Fourier transform is defined by:
\begin{align}
\hat{f}(l)=\left\langle \chi_l , f \right\rangle=\sum_{n=1}^N \chi_l^{\ast} (n) f(n),
\end{align}
with $N$ the number of vertices. We can also use a vector expression:
\begin{align}\label{eq:fou}
\hat{f}={\bf X}^T f,
\end{align}
where $\bf X$ is the matrix whose columns are equal to the eigenvectors of the Laplacian graph $\textbf{L}$ and $X^T$ the transpose of a matrix $X$.

The construction of the elements of the basis of spectral graph domain from the eigenvalues and the eigenvectors therefore leads to the definition of the wavelet transform on graphs\cite{ar:hammond}.

\subsubsection{The spectral graph wavelet  transform}\label{ssec:sgwt}
The SGWT is an approach of a multiscale graph-based analysis which is defined in the spectral domain \cite{ar:hammond}. The scale notion for the graph domain is defined by analogy with the representation of the classical wavelet transform in the spectral domain.

The graph wavelet operator at a given scale $t$ is defined by $T_g^t=g(t \mathbf{L})$ with $g(x)$ stands for the wavelet kernel which is considered as a band pass function: $g(0)=0$,\ $\lim_{x\rightarrow \infty} g(x)=0$. 
The wavelet coefficient $W_f$ at the scale $t$ has the following expression\cite{ar:hammond}:
\begin{align}\label{eq:wv}
W_f(t,n) &=T_g^t (f)(n)=\left(g(t\textbf{L}) f\right)(n)\\
&=\sum_l g(t\lambda_l) \hat{f}(l) \chi_l(n),
\end{align}
whereas the scaling function coefficient $S_f$ is given by:
\begin{align}\label{eq:2}
S_f(n) = \left(h(\textbf{L})f\right) (n)=\sum_l h(\lambda_l) \hat{f}(l) \chi_l(n),
\end{align}
where $h(x)$ represents a low pass function satisfying $h(0)>0$, and $h(x)\rightarrow 0 $ as $x\rightarrow \infty$.

The wavelet and the scaling kernels have to satisfy some conditions  \cite{ar:hammond}. In order to simplify the calculations,  we chose to implement the cubic spline kernel. To facilitate the use of this transform, the Chebyshev polynomial approximation is suggested by the authors \cite{ar:hammond}. One can notice that this transform is invertible.

This transform provides a general framework for adapting the data representation. Indeed one may observe that one of the central elements of the construction of this representation will be of course the graph and more specifically the structure that is associated with the input data. This transform is also well suited for introducing psychovisual information.

We will now introduce our strategy for graph construction. The objective is to build a graph with a model organized in such a way as to measure the color similarities.
\section{Graph construction from the color image}\label{sec:33}
In this section we develop an adaptive approach where the study is not restricted to spatial coordinates. The idea is to be able to compute the distances between a pairwise of vertices which can be used to further characterize the color structure.  In low-dimensional embedding, the geodesic distance has proven its efficiency to estimate the intrinsic geometry of the data manifold used for example in nonlinear data analysis (Isomap technique \cite{ar:isomap}). Traditionally geodesics are defined through a parallel transport of a tangent vector in a linear connection. This distance is adequate for our definition in order to model the topological structure of the data rather well. In the next section, we note $(x_m,y_m)$ the embedding of the graph $G$ of the image on 2-dimensional Euclidean space with $m$ the vertex, $(R_m,G_m,B_m)$ the embedding on 3-color Red-Blue-Green space and $(L^\ast_m,a^\ast_m,b^\ast_m)$ the embedding on CIELAB space.

\subsection{Graph connection}
The computation of the geodesic distance is estimated by applying an Isomap strategy \cite{ar:isomap} which consists in finding first the $k$-Nearest Neighbors ($k$-NN) for each pixel having the coordinates $(x,y,r,g,b)$. In order to find the shortest paths between two vertices and thus two color pixels we used the Dijkstra algorithm\cite{dijkstra1959note}. 
 The goal is to find the shortest path from the source vertex to the other vertices in the graph. 
 
To properly identify the geodesic distances the choice of the number of the $k$-Nearest Neighbors must not be too high \cite{Meng2008862}. Indeed some abnormal neighborhood edges deviate from the underlying manifold ({\em short circuit} effect). Moreover this value should be as low as possible in order to minimize the calculation time.
Ideally, we should compute all possible cases of the geodesic distance but for larger data, this requires significant computing and memory capacity. To achieve this, we consider only a small percentage of the totality of the structure. 

The weighted graph defined by equation (\ref{eq:we}) is thus obtained by applying the function, with $\sigma>0$ and 
the geodesic distance $\rho_{mn}$ between vertices $m$ and $n$ is defined as follows:
\begin{align}
\rho_{mn}=d_{m i}+d_{ij}+\cdots+d_{kn},
\end{align}

with $m \rightarrow i$, $i \rightarrow j$ $\cdots$ $k \rightarrow n$ the shortest path connecting two vertices according $d_{ij}$ the path distance connecting two vertices. The choice of distance is discussed in \ref{ssec:ED} and \ref{ssec:deltaE}. 

The weight function attempts to measure the difference between vertices and provides a clear link with the graph connection and the smoothness of the signal on the graph. If we choose an important value of $\sigma$ this leads to add more edges to the graph and also increases the total variation because two different values of the signal on the graph will probably be connected as discussed below.   

We now wish to discuss how our proposed framework can yield a graph wavelet for a color image by inserting a color dimension into the measured distance. Two approaches are put forward: in the first we are interested in simply representing the data in the classical Red Green Blue (RGB) color space, in the second we have used the CIELab space which can bring a visual aspect to the analysis. With this approach, we are interested in highlighting the visual characteristics in order to have a relevant representation of the color data through a weighted graph. Such a strategy allows us to enhance the link between two similar colors and to clearly indicate the ruptures of the color regions.

\subsection{Euclidean distance on RGB color space}\label{ssec:ED}

We denote $\Delta E_d^{mn}$ the Euclidian color distance between a pair of color values in RGB space $(R^\ast_m,G^\ast_m,B^\ast_m)$ and $(R^\ast_n,G^\ast_n,B^\ast_n)$ of both vertices $m$ and $n$. The $k$-NN distance between two vertices $m$ and $n$ is also defined by: 
\begin{align}\label{eq:c}
d_{mn}^2=(x_m-x_n)^2+(y_m-y_n)^2+ \left(\Delta E_d^{mn}\right)^2.
\end{align}

Equation (\ref{eq:we}) is then used to compute the weights of the edges. This way of building the graph allows us to consider all the color components of the image. Transforms are applied on each color component which is represented as a signal on the graph.
In figures \ref{fig:de5} and \ref{fig:de30} we show an example of the application on a color image using two different values of $\sigma$ $(\sigma_1=5, \sigma_2=30)$, the scale number $J$ is set to $3$. The approximation result in figure \ref{fig:de30} with a high $\sigma$ has more homogeneous areas than the approximation result in figure \ref{fig:de5} and also more details in the scales. It is to be noted that the choice of $\sigma$ remains empirical. We indicate on these figures a value $q_\tau$, the mean of quadratic form of wavelet coefficient on each scale $\tau$ for each color component and  a value $q$, the mean of quadratic form of wavelet coefficient for each color component. 
 It should be noted that the quadratic form is introduced in the following paragraph.

Now we propose a model of the behavior of the information in the SGWT basis. Parameter $\sigma$ has some influence on the analysis revolving around two main observations:
\begin{itemize}
\item the nature of details does not decrease according to the wavelet scale;
\item when $\sigma$ increases more information is contained in higher frequency bands. 
\end{itemize}

To clarify these issues, we focused on the distribution of the energy through the wavelet scale coefficients. Possibly the simplest way to measure this energy is to rely on the notion of the global smoothness by measuring the Laplacian quadratic form $q$.

\begin{figure}[!htb]
\centering
\includegraphics[width=7cm]{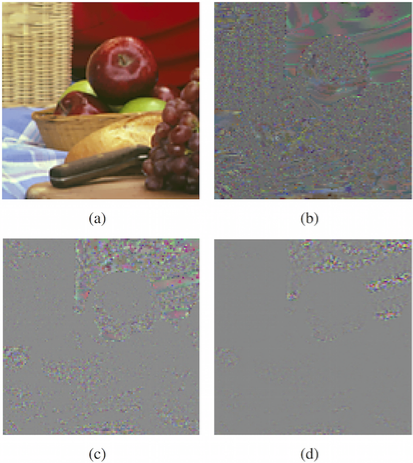}
\caption{SGWT on image \ref{fig:imagetest}(a) using the {\em ED}, $J=3$, $\sigma=5$, $q=366.6$: (a) approximation ($q_0=339$), (b) scale 1 ($q_1=191$), (c) scale 2 ($q_2=187$) and (d) scale 3 ($q_3=30$).}
\label{fig:de5}
\end{figure}

\begin{figure}[!htb]
\centering
\includegraphics[width=7cm]{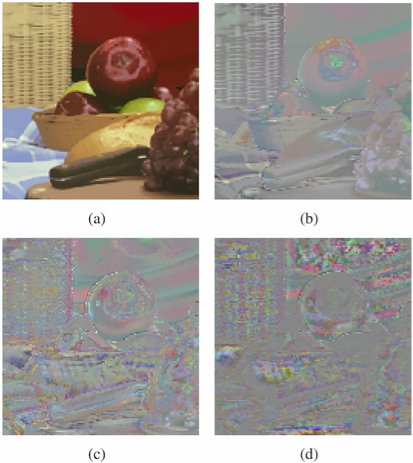}
\caption{SGWT on image \ref{fig:imagetest}(a) using the {\em ED}, $J=3$, $\sigma=5$, $q=6821.2$: (a) approximation ($q_0=6034$), (b) scale 1 ($q_1=3199$), (c) scale 2 ($q_2=3931$) and (d) scale 3 ($q_3=601$).}
\label{fig:de30}
\end{figure}

The quadratic form is defined as follows \cite{ar:rus} \cite{ar:spel}:
\begin{align}\label{eq:q}
q=\sqrt{f^T \textbf{L} f},
\end{align}
where $f$ is the signal on the graph. The quadratic form is small when $f$ is similar at the two extremes of the edge with a high weight \cite{ar:shumi}. 

Usually, the energy of information of the classical wavelet transform increases with the scale. Consequently the quadratic form increases also with the scale.

\begin{proposition}\label{prop:1}
The quadratic form of a given signal on a graph $f$ can be written as follows:
\begin{align}
q&=\sqrt{\sum_{l=0}^{N-1} (\hat{f} (l) \sqrt{\lambda_l})^2}.
\end{align}
\end{proposition}

\begin{proof}
Let us consider the expression of the quadratic form shown in (\ref{eq:q}).
$\textbf{L}$ is defined as follows:
\begin{align}\label{eq:ll}
\textbf{L}&=\textbf{X} \mathbf\Lambda \textbf{X}^T,
\end{align}
where $\mathbf\Lambda=diag(\lambda_0,\lambda_1,\lambda_2,\cdots,\lambda_{N-1})$ is the matrix of eigenvalues of $\textbf{L}$, substituting (\ref{eq:ll}) in (\ref{eq:q}):
\begin{align}
q^2&=f^T \textbf{X} \mathbf\Lambda \textbf{X}^T f,\\
q^2&=(\textbf{X}^T f)^T \mathbf\Lambda (\textbf{X}^T f).
\end{align}
According to equation (\ref{eq:fou}) we have:
\begin{align}\label{eq:piu}
q^2&= \hat{f}^T \mathbf\Lambda \hat{f}=\sum_{l=1}^N (\hat{f} (l) \sqrt{\lambda_l})^2.
\end{align}
\end{proof}

Writing the quadratic form in the spectral domain is a good representation of the variation of the signal on the graph, hence the quadratic form has a significant value when the based-graph Fourier transform coefficients on the high frequencies are high.

Let $p_\tau$ be the wavelet kernel at the scale $\tau$ with $p_0$ the scaling function kernel. The based-graph Fourier transform of the wavelet coefficients at each scale is written:
\begin{align}\label{eq:po}
\hat{f}_\tau (l)= p_\tau(l) \hat{f} (l),
\end{align}  
with $f$ the original signal on the graph. According to proposition \ref{prop:1}, the quadratic form of the wavelet coefficients at each scale $\tau$ is:
\begin{align}\label{eq:pu}
q_\tau^2&=\sum_{l=1}^N  p_\tau^2 (\lambda_l)  \hat{f}^2 (l) \lambda_l.
\end{align}
The localization of the information follows equation (\ref{eq:pu}) defined in the spectral domain.

\begin{figure}[htb]
\centering
\includegraphics[width=7cm]{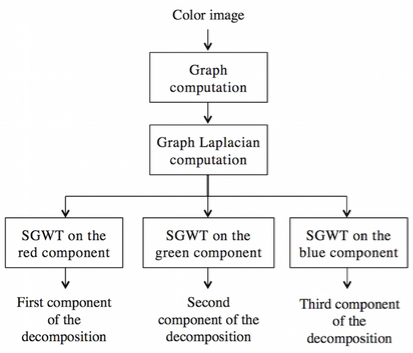}
\caption{Flowchart of a wavelet decomposition on graph of a color image.}
\label{fig:flowchart}
\end{figure}

A color image is decomposed in 3 scales using the numerical scheme in figure \ref{fig:flowchart}.
In figure \ref{fig:de5} one can observe clearly more details at scale 1 then scale 2 and in figure \ref{fig:de30} one can see clearly more details at scale 2 then scale 1.

 This is verified by comparing their quadratic form values ($q$ equals the average of the quadratic form of all color components) as presented in figures \ref{fig:de5}  and \ref{fig:de30}.

As explained before, $\sigma$ involves modifications in adjacencies in the graph structure. In particular, that has a considerable impact on the variation of the information content. This phenomenon is measured in our framework by the quadratic form which increases with the number of connections in the weighted graph. This involves more localization of the energy of the signal $f$ in the higher frequencies (see more details in \cite{ar:shumi})  as one can see in figure \ref{fig:de30}. 

One can observe that the SGWT extracts various high level information or details to solve a specific problem. For example, in the case of image denoising application, the choice of $\sigma$ and the color distance is adapted to obtain a significant separation of noise as will be see in section \ref{sec:5}.

We now explore the specifications of the graph construction through the geodesic distance to introduce a psychovisual conception in the graph wavelet.

\subsection{The $\Delta E_{2k}$ difference }\label{ssec:deltaE}
The use of the visual features is becoming increasingly widespread in the field of image processing. Several metrics developed to assess the perceptual qualities of color images falls within this context. The principle is thus based on the modeling of the human visual system. 
The CIELab color space is considered as the most complete color space which describes all colors visible to the human eye. This representation is created to serve as a device independent model used as a reference.
The three coordinates of CIELab represent the lightness of the color ($L^\ast$), the position between red and green ($a^\ast$) and the position between yellow and blue ($b^\ast$).

Intuitively, the employment of these characteristics in the analysis of the data provides a good localization of the specific means of the color which can be adapted to the human vision. The purpose of this work is to introduce this aspect on the multi-scale analysis, therefore we compare the features of a given pair of vertices by using the $\Delta E_{2k}$ color difference. The $\Delta E_{2k}$ difference allows us to measure the difference between two colors modeled on the psychovisual system  \cite{ar:de}. Moreover, a high $\Delta E_{2k}$ difference corresponds to an important colorimetric difference (see \cite{ar:zhang,ar:luo,ar:vienna}). 

This notion can be included into the transform through the computation of the geodesic distances of the data represented in the CIELab space. This amounts to combining spatial and color comparisons. As regard to the color difference, we opt for the $\Delta E_{2k}$ difference.

We denote $\Delta E_{2k}$ the color difference between a pair of color values in CIELab space $(L^\ast_m,a^\ast_m,b^\ast_m)$ and $(L^\ast_n,a^\ast_n,b^\ast_n)$ of both pixels $m$ and $n$ as follows:
\begin{align}
\Delta E_{2k}(L^\ast_m,a^\ast_m,b^\ast_m,L^\ast_n,a^\ast_n,b^\ast_n)=\Delta E_{2k}^{mn},
\end{align} 
the $k$-NN distance between two pixels $m$ and $n$ is also defined by: 
\begin{align}\label{eq:c}
d_{mn}^2&=(x_m-x_n)^2+(y_m-y_n)^2+ \left(\Delta E_{2k}^{mn}\right)^2, 
\end{align}

We illustrate the decomposition based on the $\Delta E_{2k}$ difference and by using the same parameters as the previous section. 
For all of computations, $\Delta E_{2k}$ differences are computed for an observer visualizing images 47 cm away from a LCD monitor displaying 72 dots per inch (a standard configuration). 
So a $\Delta E_{2k}$ greater than 4 corresponds to a high perceptual difference and it is thus not necessary to weight the perceptual difference and the spatial distance.

The influence of parameter $\sigma$ is illustrated in both figures \ref{fig:deltaE1} ($\sigma=5$) and \ref{fig:deltaE2} ($\sigma=30$). Indeed when $\sigma$ increases, details shift to the high frequencies.

If we compare the difference with the classical Euclidean distance, at a fixed value of $\sigma$, the computation of the distance measured in the CIELab space creates obviously more edges and enhances the links between neighbors. The quality of the real color distance has direct impact on the increasing quadratic form: the energy of the based-graph signal ${f}$ is more concentrated on the higher frequencies. From proposition \ref{prop:1}, we deduce that the values of the coefficients of the based-graph Fourier transform $\hat{f} (l)$ in the higher frequencies increase in relation to the weights and the number of edges. To illustrate this point, we obtain a larger value of the quadratic form for color distance (figure \ref{fig:deltaE2}, $q=20507$), than for Euclidean distance (figure \ref{fig:de30}, $q=6828.2$).

Consequently, the change of the distance implies a transfer of details to the higher frequencies, as one can see when comparing figures \ref{fig:de30} and \ref{fig:deltaE2} (or figures \ref{fig:de5} and \ref{fig:deltaE1} ). In figure \ref{fig:de30} the scaling function coefficient has the higher quadratic form ($q_0=6034.4$) while in figure \ref{fig:deltaE2} the higher quadratic form corresponds to the wavelet coefficient on the scale $2$ ($q_2=15499$). This phenomenon is explained in the equation \ref{eq:pu}: the $\Delta E_{2k}$ difference causes an increase of the Fourier coefficients on the higher frequencies, and this increase is controlled by the wavelet kernel at each scale noted ($p_\tau(\lambda)$).

To conclude, figure \ref{fig:deltaE1}(a) highlights the importance of taking into account the human visual system. Indeed color regions appear notably more contrasted than figures \ref{fig:de5} and \ref{fig:de30} and one can observe also a good separation with sharper edges. In figure \ref{fig:deltaE2}(a) the colored regions become more homogeneous because of the effect of parameter $\sigma$.

\begin{figure}[!ht]
\centering
\includegraphics[width=7cm]{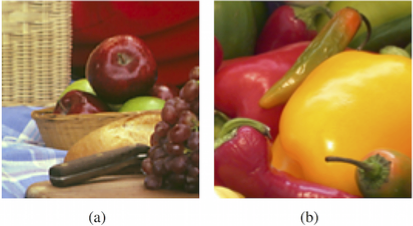}
\caption{Images for experimentations: (a) image 1, (b) image 2.}
\label{fig:imagetest}
\end{figure}

\begin{figure}[!ht]
\centering
\includegraphics[width=7cm]{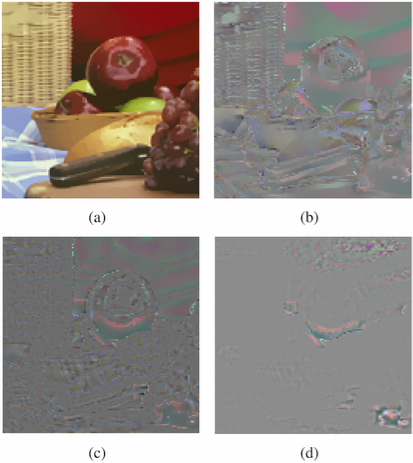}
\caption{SGWT on image \ref{fig:imagetest}(a) using the  $\Delta E_{2k}$ distance, $J=3, \sigma=5,q=2661$: (a) approximation ($q_0=1888$), (b) scale 1 ($q_1=1171$), (c) scale 2 ($q_2=1706$) and (d) scale 3 ($q_3=506$).}
\label{fig:deltaE1}
\end{figure}

\begin{figure}[!ht]
\centering
\includegraphics[width=7cm]{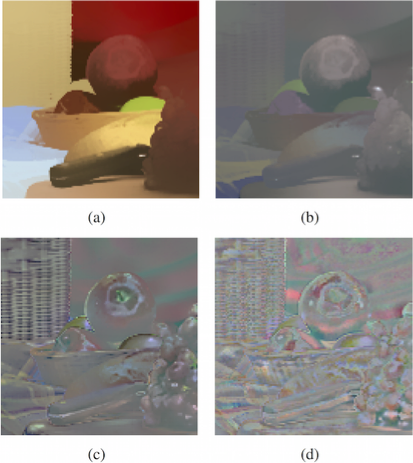}
\caption{SGWT on image \ref{fig:imagetest}(a) using the   $\Delta E_{2k}$ distance, $J=3, \sigma=30,q= 10851$: (a) approximation ($q_0=1888$), (b) scale 1 ($q_1=7440$), (c) scale 2 ($q_2=15499$) and (d) scale 3 ($q_3=3691$).}
\label{fig:deltaE2}
\end{figure}

\subsection{Conservation of the data organization into the basis}
So far, in order to illustrate the importance of the phase information in the Fourier transform, many authors have proposed to mix the phase of a first image with the magnitude of the Fourier coefficients of a second image. In this subsection, in order to illustrate how our scheme takes into account the geometry of the data and captures singularities of the image, we suggest a similar operation. We will construct the graph from a first image and compute the analysis with the SGWT on a second image. In our example we use the image \ref{fig:imagetest}(a) to construct the graph, then the SGWT is applied on the image \ref{fig:imagetest}(b). Figures \ref{fig:geomeuc}, \ref{fig:geomlab} illustrate respectively the representation in the RGB and CIELab color space.

It is to be noted that the geometry is recovered into the scaling function coefficients. We set the parameter of deviation $\sigma$ at a large value in order to minimize the information content of image \ref{fig:imagetest}(b) at the scaling function coefficients. Two observations can be made:
\begin{itemize}
\item Figure \ref{fig:geomeuc}(a) shows that the Euclidean distance is in better accord with textures (on the left) because this distance designed in the RGB space recognizes well the data organization. Indeed textures are characterized by energy measures:  the perceptual approach is therefore not appropriate.
\item Figure \ref{fig:geomlab}(a) shows that the  $\Delta E_{2k}$ color difference conserves edges well and keeps more color regions. We also see a preservation of the specular region on objects presented on the original image \ref{fig:imagetest}(a).
\end{itemize}

\begin{figure}[!htb]
\centering
\includegraphics[width=7cm]{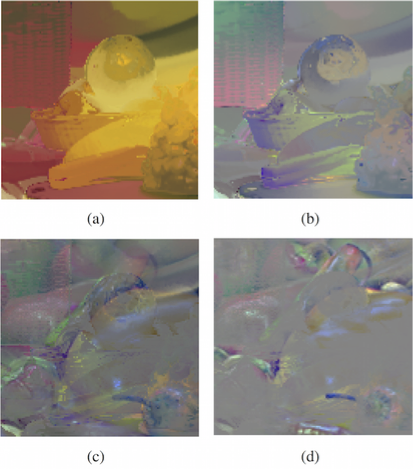}
\caption{SGWT with the graph constructed from image \ref{fig:imagetest}(a) and applied on image \ref{fig:imagetest}(b) using the {\em ED}, $J=3, \sigma=30$: (a) approximation, (b) scale 1, (c) scale 2 and (d) scale 3.}
\label{fig:geomeuc}
\end{figure}

\begin{figure}[!htb]
\centering
\includegraphics[width=7cm]{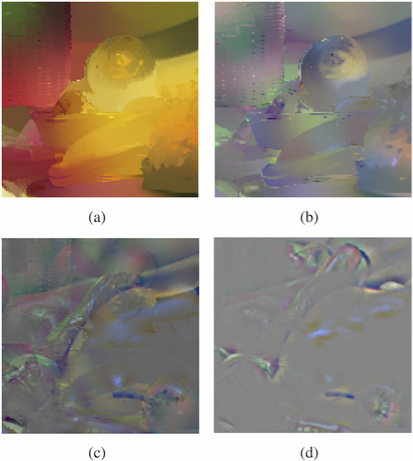}
\caption{SGWT with the graph constructed from image \ref{fig:imagetest}(a) and applied on image \ref{fig:imagetest}(b) using the $\Delta E_{2k}$ difference, $J=3, \sigma=10$: (a) approximation, (b) scale 1, (c) scale 2 and (d) scale 3.}
\label{fig:geomlab}
\end{figure}

The experiment above has shown the advantages of the implementation of perceptual features in the analysis of the color data: the color information is well encoded into the wavelet basis. 
\section{Potential of this color transform for image restoration}\label{sec:5}
Image denoising is one of the fundamental applications of the multiscale analysis. The representation on a graph domain preserves the structure of the data and it ensures also the robustness of the denoising process. We assess the performance of our approach with the two representations of the color data.

To compare our proposed denoising methods in the same context (multiscale transform and thresholding), we compute the classical denoising method based on an undecimated D8 Daubechies wavelet transform (UWT) with three scales.
We have used the undecimated Wavelet transform because this decomposition obtains better denoising results than the classical orthogonal Wavelet thresholding and the SWGT is also a redundant transform.
 The SGWT is computed with a cubic spline kernel as specified in \ref{ssec:sgwt}. 

\subsection{Comparison of image denoising methods}

The denoising procedure by both wavelet transforms simply consists in thresholding the wavelet coefficients on each color component and computing the inverse transform. 
Hard thresholding\cite{donoho1994ideal} of a wavelet coefficient $W_f(t,n)$ is defined by: 
\begin{equation}
\mathcal T^H_{\alpha}(W_f(t,n)) = W_f(t,n) 1_{( |W_f(t,n)|>\alpha)},
\end{equation}
with $\alpha$ the threshold. In our application, $\alpha$ is the empirical threshold defined as $\alpha=3\tau$ with variance $\tau$ which is estimated using the absolute median of the wavelet decomposition's first scale.

The denoising processing with the perceptual graph wavelet transform consists in algorithm \ref{algo:denoising}. A flowchart is proposed in figure \ref{fig:flowchart_denoising}. We note that geometrical and color structuration is contained in graph $G$. 

\begin{algorithm}[!ht]
Initialization: Color Image $I$\;
$I_\text{smooth}$ $\leftarrow$ Smooth($I$)\;
Compute graph $G$ of $I_\text{smooth}$\;
Compute $\mathbf L$\;
\For{each color component}{
Compute SGWT on $I$ with $L$\;
Compute hard thresholding\;
Compute inverse SGWT\;
}
\caption{Graph wavelet denoising}
\label{algo:denoising}
\end{algorithm}

\begin{figure}[htb]
\centering
\includegraphics[width=7cm]{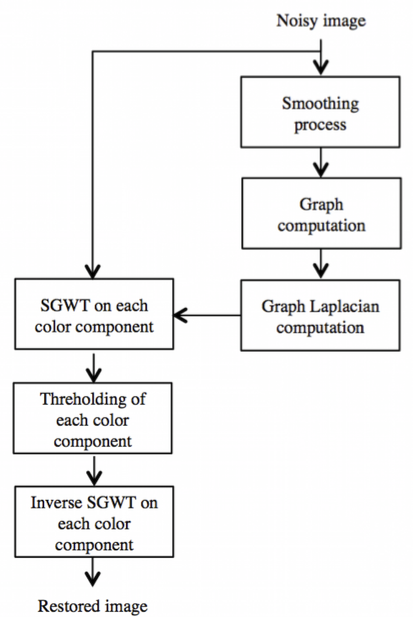}
\caption{Flowchart of the denoising with the wavelet decomposition on graph of a color image.}
\label{fig:flowchart_denoising}
\end{figure}

We simulate noisy images by adding a Gaussian white noise to a color image with a known standard deviation and we then compute the image recovered from the noisy one by each method.

One of the drawbacks of the SGWT is the poor denoising performance when applied directly on noisy images \cite{ic:hammond2012}. This poor separation of noise is due to the fact that the scaling function is arbitrary and not scale-dependent \cite{ar:mining}. In figure \ref{fig:decbruit} we show an example of direct application of the SGWT on noisy images. We have a degraded quality of the wavelet coefficient at each scale and the coefficients provide limited details. Therefore, such a decomposition depends heavily on the constructed graph, we need a regularized graph which will allow us to estimate the image structure and provides with the wavelet coefficients a good separation between the noise and the information. 
\begin{figure}[!ht]
\centering
\includegraphics[width=7cm]{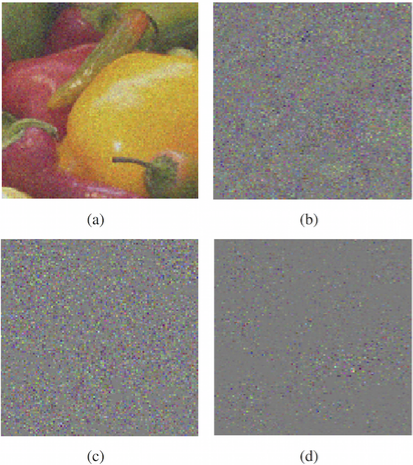}
\caption{SGWT applied on a noisy image \ref{fig:imagetest}(a) using the {\em ED}, $\sigma=10$: (a) approximation, (b) scale 1, (c) scale 2 and (d) scale 3.}
\label{fig:decbruit}
\end{figure}

We recall that our aim through this study is to illustrate the advantages of the representation in the CIELab color space. To regularize coarsely the noisy image, we apply a Gaussian smoothing and we construct the graph on this result. We then compute an analysis with the SGWT on each color component (RGB) of the noisy image using a cubic spline wavelet kernel with three scales ($J=3$). 

It is to be noted that we did not choose the same value of $\sigma$ for the graph construction to compare the results of reconstructions. As we have seen previously for the same value of $\sigma$, the $\Delta E_{2k}$ difference allows us to localize more energy in the higher frequencies : this can affect the quality of the reconstructed image. One can address  this issue by choosing a lower value of $\sigma$.

Several techniques of estimation of the clean image and wavelet coefficients thresholding are also discussed in \cite{ic:hammond2012}.  

\subsection{Results interpretation}
We discuss the denoising schemes with two color images {\em girl} and {\em butterfly} which propose different textures and colors, specular effects, round objects, edges in different directions, etc.

\begin{figure}[!ht]
\centering
\includegraphics[width=7cm]{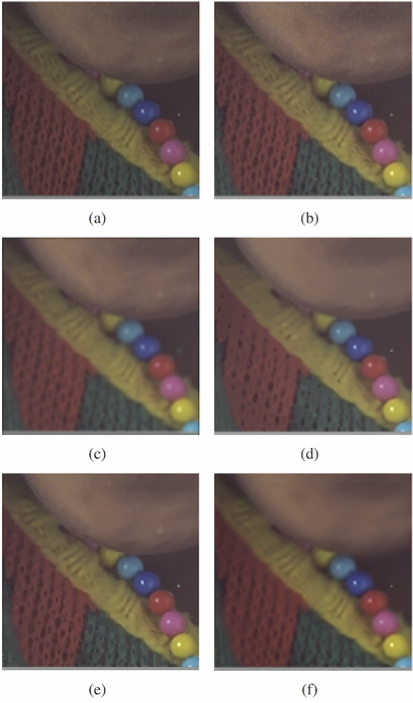}
\caption{Denoising results for image {\em Girl} with a low noise: 
(a) Original image;
(b) Noisy image (SNR=(3.4,5.1,3.9)dB, SSIM=0.965, QSSIM=0.944);
(c) Smoothed image (SNR=(3.3,5.0,3.8)dB, SSIM=0.924, QSSIM=0.922);
(d) {\em ED}, $\sigma=10$ (SNR=(1.8,3.1,1.6)dB, SSIM=0.916, QSSIM=0.905);
(e) $\Delta E_{2k}$, $\sigma=1$ (SNR=(4.2,6.3,5.4)dB, SSIM=0.947, QSSIM=0.939);
(f) UWT (SNR=(3.3,5,3.8)dB, SSIM=0.931, QSSIM=0.927).
}
\label{fig:1g}
\end{figure}

\begin{figure}[!ht]
\centering
\includegraphics[width=7cm]{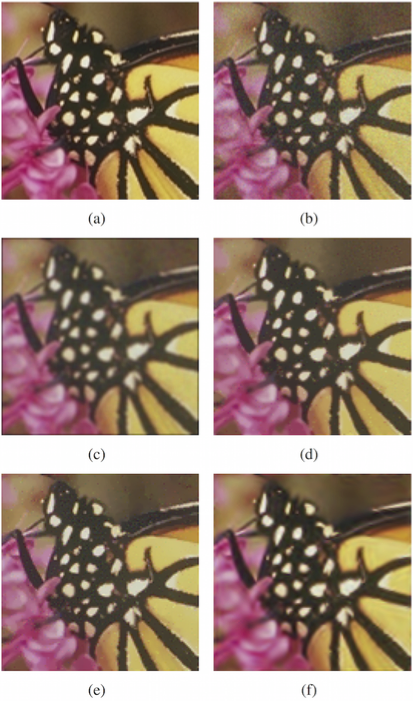}
\caption{Denoising results for {\em butterfly} image with a low noise:
(a) Original image;
(b) Noisy image (SNR=(13.4,12.2,9.9)dB, SSIM=0.903, QSSIM=0.850);
(c) Smoothed image (SNR=(9.9,9.4,7.9)dB, SSIM=0.861, QSSIM=0.853);
(d) {\em ED}, $\sigma=10$ (SNR=(14.0,12.7,10.3)dB, SSIM=0.918, QSSIM=0.898);
(e) $\Delta E_{2k}$, $\sigma=2$ (SNR=(14.0,12.9,10.5)dB, SSIM=0.928, QSSIM=0.907);
(f) UWT (SNR=(12.6,11.3,9.1)dB, SSIM=0.921, QSSIM=0.916).
}
\label{fig:1b}
\end{figure}

\begin{figure}[!ht]
\centering
\includegraphics[width=7cm]{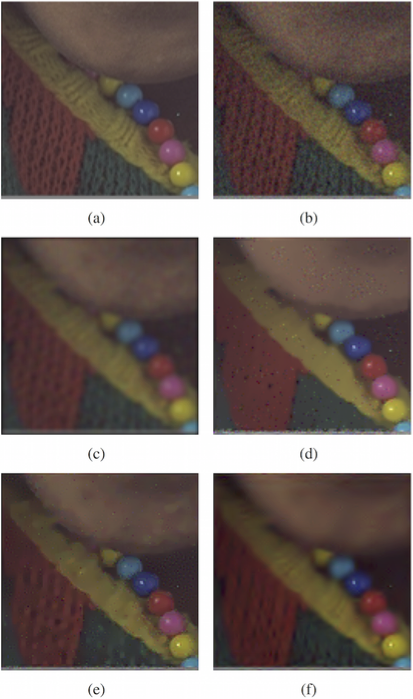}
\caption{Denoising results for image {\em Girl} with a high noise:
(a) Original image;
(b) Noisy image (SNR=(-0.8,-1.1,-3.3)dB, SSIM=0.733, QSSIM=0.571);
(c) Smoothed image (SNR=(-0.2,-0.5,-2.7)dB, SSIM=0.830, QSSIM=0.823);
(d) {\em ED}, $\sigma=10$ (SNR=(-3.8,-3.1,-4.9)dB, SSIM=0.792, QSSIM=0.740);
(e) $\Delta E_{2k}$, $\sigma=2$ (SNR=(0.6,0.6,-1.5)dB, SSIM=0.810, QSSIM=0.763);
(f) UWT (SNR=(-0.2,-0.5,-2.7)dB, SSIM=0.817, QSSIM=0.819).
}
\label{fig:2g}
\end{figure}

One can note that this quality of restoration is confirmed by the measurement of SNR (we have noted the SNR measurements on the red, green and blue bands as $\mbox{SNR}=\left(\mbox{SNR}_{\mbox{red}},\mbox{SNR}_{\mbox{green}},\mbox{SNR}_{\mbox{blue}}\right)$dB.

The structural similarity (SSIM) index \cite{wang2004image} is computed to measure the quality of the reconstruction of the original geometrical and structural information. This measure is associated with the Value component of the color image. We compute also the quaternion structural similarity index (QSSIM) \cite{kolaman2012quaternion} that takes into account the quality of the reconstruction of the color information.

In figures \ref{fig:1g}(e) and \ref{fig:1b}(e), we observe that the restoration in CIELab color space provides satisfactory results with a low level  of noise. Indeed, the psychovisual consideration improves the visual quality of noisy image and ensures the conservation of the discontinuity between the homogeneous regions. One can observe that the reconstruction in the RGB color space in figures \ref{fig:1g}(f) and \ref{fig:1b}(f) indicates a good result but a lot of details of the image have disappeared. Comparing the UWT method in figures \ref{fig:1g}(f) and \ref{fig:1b}(f), we observe that important details are preserved (such as the reflection on the shield) with our method. The graph structuration allows us to connect vertices similar in terms of color distance or perceptual difference. 

This indicates that the classical measures are not always appropriate for estimating the degradation: the SSIM and QSSIM are greater in the noisy image in figure \ref{fig:1g}(b), because these measures are particularly sensitive to very small variations and the low noise does not enough damage the structure of our original image in figure \ref{fig:1g}(a). However the image denoising with our method using the psychovisual information in figure  \ref{fig:1g}(e) is better in terms of SSIM, QSSIM and SNR than the reference denoising using UWT in figure \ref{fig:1g}(f).

When we increase the noise level (figure \ref{fig:2g}(b)), it is difficult to construct a graph without a good estimation of the coarse approximation of the original image. This is due to the problem of the estimation of the initial graph. A simple Gaussian smoothing  (figure \ref{fig:2g}(c)) does not produce a good noise-free coarse representation of data and one can observe a cartoon-like effect on the homogeneous zones of the denoised image in figures \ref{fig:2g}(d) and \ref{fig:2g}(e). 

It is to be noted that there is no global perceptual difference between the UWT and the SGWT defined with The Euclidian distance. Both these restorations are however different from one antoher with a better enhancement of edges for the homogeneous regions for the SGWT method (figure \ref{fig:2g}(d)) in comparison to the UWT method (figure \ref{fig:2g}(f)). 

\begin{figure}[!ht]
\centering
\includegraphics[width=7cm]{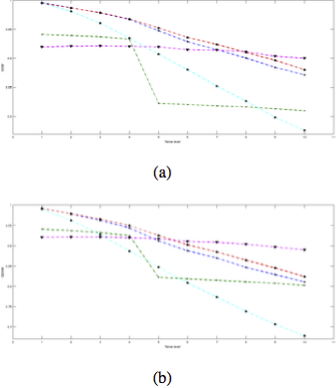}
\caption{(a) Structural similarity index and (b) quaternion structural similarity index in function the noise standard deviation for each method (noisy image in cyan, SGWT with {\em ED} in blue, SGWT with $\Delta E_{2k}$ in red, smoothed image in green and UWT in magenta).}
\label{fig:ssimcomparaison}
\end{figure}

We propose to assess the quality of methods by measuring structural similarity indices (SSIM and QSSIM) between denoised images and the original one for different levels of noise. 
In figure \ref{fig:ssimcomparaison}(a) the SSIM highlights that the method based on the SGWT with the $\Delta E_{2k}$ difference obtains a better reconstruction than the same method with the {\em ED}. 
In figure \ref{fig:ssimcomparaison}(b) the QSSIM highlights the same results. We observe that the method based on the SGWT with the $\Delta E_{2k}$ difference obtains a best perceptual reconstruction for a slight noise. When the standard deviation of noise is greater than 5.5, the UWT D8 obtains a better perceptual reconstruction compared to our based-graph decomposition. 
Indeed, one cannot obtain a good first estimation of the based-graph image after a simple regularization when the noise is too high. It will be necessary in the future to define a better strategy to estimate the based-graph coarse approximation of the image.

\subsection{Potential of these color transform for inpainting}\label{sec:6}
We propose to extend the scope of applications to the inpainting in color images. We extend the principle of inpainting with classical wavelet on the graph domain. Our approach consists in the use of a classical thresholding iterative process. 

Let $f$ be the observed signal and $y=\Phi v +n$ the signal with missing pixels at locations $\Phi$, $n$ an additive noise.
The estimation $\hat y$ can be obtained by resolving the following optimization:
\begin{align}
\hat y=\argmin_v\frac{1}{2}\left\| y-\Phi v \right\|_2^2 + \Lambda R(v),
\end{align}
where $R$ is a regularizer which imposes an {\it a priori} knowledge and $\Lambda$ should be adapted to the noise level.

By definition, $v$ is sparse in the dictionary $W^\ast$ (synthesis operator) and $v$ is obtained by linear combination $v=W^\ast a$. We compute a sparse set of coefficients denoted $a$ in a frame $W=(\psi_m)$:
\begin{align}
\hat a=\argmin_a\frac{1}{2}\left\| y-\Phi \Psi^\ast a  \right\|_2^2 + \Lambda \ R(a),
\end{align}
with $R(a)$ the $\ell_1$ sparsity regularizer defined by:
\begin{align}
R(a)=\sum_m \left\|a_m\right\|_2.
\end{align}

The problem corresponds to a linear ill posed inverse problem. It can be solved by an iterative scheme with the hard thresholding operator $\mathcal T^{H}_\alpha$.
In the case of the SGWT, the wavelet transform is computed in respecting the structure of data, the solution is also obtained by the numerical scheme presented in algorithm \ref{algo:2}.
\begin{algorithm}[!ht]
Initialize $N$\\
 $v^0 \leftarrow y$; $k \leftarrow 0$\;
Compute $\mathbf L$\;
\While{￼$k<N$}
{
$T_g^t \leftarrow g(tL)$\;
$ v^{k+1} \leftarrow v^k + \Phi\left[\left(T^t_g\right)^\ast \mathcal T^H_\alpha \left(T^t_g\right) v^k - v^k \right]$\;
$ k \leftarrow k+1$\;
}
\caption{Graph wavelet inpainting}
\label{algo:2}	
\end{algorithm}

The graph cannot be constructed for damaged zones: we should therefore provide graph construction techniques more suitable for missing data.
\begin{figure}[!ht]
\centering
\includegraphics[width=7cm]{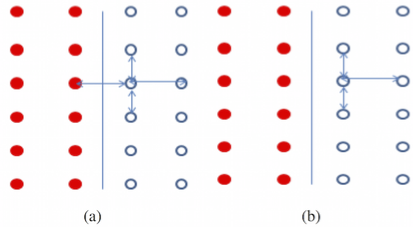}
\caption{Topology on edge with missing data: (a) classical, (b) with kppv.}
\label{fig:topo}
\end{figure}

To connect vertices in the spatial domain with edges, a regular mesh represents a simpler tool to estimate the data structure of the missing areas (cf figure \ref{fig:topo}). 
From this principle we propose to introduce a regular mesh within the processing areas.
\begin{figure}[!ht]
\centering
\includegraphics[width=7cm]{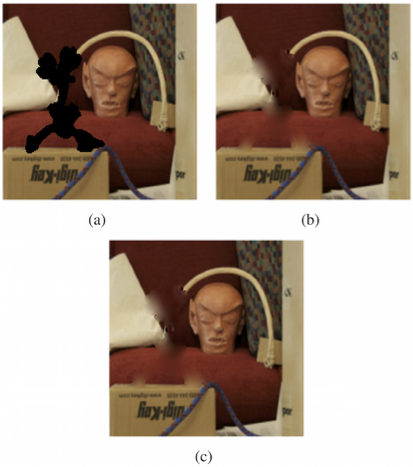}
\caption{Inpainting on an attacked image:
(a) Attacked image (SNR=$(3.2, 1.9, 12.4)$dB, SSIM=0.88, QSSIM=0.89);
(b) {\em ED} (SNR=$(10.8, 12.2, 16.0)$dB, SSIM=0.93, QSSIM=0.92;
(c) $\Delta E_{2k}$ distance (SNR=$(10.8, 14.1, 16,0)$dB, SSIM=0.93, QSSIM=0.93).}
\label{fig:inpainting}
\end{figure}

We have applied the graph-based wavelet inpainting method on the image in figure \ref{fig:inpainting}(a). One can notice that the color information is diffused in the damaged regions and results show clearly the limitation of the method when applied on textured regions, that is a classical problem in the inpainting using a PDE-based method. We propose to restore this image with an inpainting processing using the RGB color space \ref{fig:inpainting}(b) and using the CIELab color space to compute the graph (figure \ref{fig:inpainting}(c)).  The graph structuration allows us to regularize similar zones whose vertices are connected.

In term of SNR, the inpainting results with a graph structure computed with an Euclidian distance and those computed with a $\Delta E_{2k}$ are close except on the green component. The quality metrics (SSIM and QSSIM) highlight a good reconstruction of structures.

\section{Conclusion}

We have introduced a framework which takes into account the intrinsic geometry and the color information of an image. Indeed we have developed a perceptual wavelet transform based on the graph of the image and, in particular, on the $\Delta E_{2k}$ color difference in the graph computation. Furthermore using the Laplacian quadratic form to measure the amount of energy of information helps understanding the behavior of this new representation. Moreover, the quadratic form depicts the influence of parameters on the localization of information across the wavelet scales. 

The application on the image denoising shows that the restoration in the CIELab space allows to enhance homogeneous zones with respect to their ruptures and keeping some important visual details. 
It is to be noted that the use of the geodesic distance allows us to preserve the geometry of the image. This provides an enhancement of regions and does not create any artifact or blur. We illustrate the respect of color in a restoration processing by introducing perceptual information in the graph computation. We have compared this new representation with an approach based on the Euclidian distance.
We have also introduced the perceptual SGWT in a classical inpainting process and the results are correct but do not depend on the two distances for graph computation.

Our work yields promising results and shows the relevance of taking into consideration the geometry of the image and perceptual information (with the CIElab space) for the wavelet transform. This methodology can be extended to a multi- or hyperspectral image where the choice of the distance to compute the graph may have a great influence on the denoising process, inpainting application and texture characterization.

 
\bibliographystyle{alpha}

\end{document}